\documentclass[letterpaper]{article} 
\usepackage{aaai2026}  
\usepackage{times}  
\usepackage{helvet}  
\usepackage{courier}  
\usepackage[hyphens]{url}  
\usepackage{graphicx} 
\usepackage{natbib}  
\usepackage{caption} 
\usepackage{algorithm}
\usepackage{algorithmic}

\usepackage{newfloat}
\usepackage{listings}
\DeclareCaptionStyle{ruled}{labelfont=normalfont,labelsep=colon,strut=off} 
\floatstyle{ruled}
\newfloat{listing}{tb}{lst}{}
\floatname{listing}{Listing}
\usepackage{latexsym}
\usepackage{amssymb}
\usepackage{amsmath}
\usepackage{amsthm}
\usepackage{booktabs}
\usepackage{enumitem}

\usepackage{color}
\usepackage[dvipsnames]{xcolor}
\usepackage[most]{tcolorbox}
\usepackage{xspace}
\usepackage{tikz}
\usetikzlibrary{positioning, arrows.meta, 3d}
\usepackage{svg}

\usepackage{tabularx}
\usepackage{array} 

\newtheorem{theorem}{Theorem}

\newtheorem{definition}{Definition}
\newtheorem{assumption}{Assumption}

\newcommand{\reals}{\mathbb{R}}
\newcommand{\inputspace}{\mathcal{F}}       
\newcommand{\latentspace}{\mathcal{Z}}      
\newcommand{\latentcomponentdim}{D}      
\newcommand{\latentindices}{\mathbf{L}}    
\newcommand{\classset}{\mathcal{C}}        
\newcommand{\prototypeindices}{\mathbf{P}} 

\newcommand{\inputvec}{\mathbf{x}}         
\newcommand{\instancevec}{\mathbf{v}}      

\newcommand{\latentrep}{\mathbf{z}}         

\newcommand{\latentrepcomp}[1]{\mathbf{z}_{#1}} 
\newcommand{\instancelatentrep}{\mathbf{y}}

\newcommand{\prototypej}[1]{\mathbf{p}_{#1}} 
\newcommand{\instactval}{\actvec(\instancelatentrep)}

\newcommand{\predictor}{\kappa}           
\newcommand{\predictorLatent}{\kappa_{\latentspace}} 
\newcommand{\encoder}{f}         
\newcommand{\basesimilarity}{\text{sim}}           

\newcommand{\explanation}{\mathcal{E}}      

\newcommand{\sigexp}{\phi_{\explanation}}

\newcommand{\actvec}{\mathbf{a}}      
\newcommand{\actveci}[1]{\actvec_{#1}}
\newcommand{\actvecexp}{\actveci{\explanation}}

\newcommand{\simvec}{\mathbf{s}}          
\newcommand{\simspaceconstr}[1]{S_{#1}} 
\newcommand{\logit}[1]{h_{#1}}          
\newcommand{\maxsimval}[1]{\overline{\actvec}_{#1}} 
\newcommand{\minsimval}[1]{\underline{\actvec}_{#1}} 
\newcommand{\simub}{\overline{\text{sim}}_{\explanation}}
\newcommand{\simlb}{\underline{\text{sim}}_{\explanation}}
\newcommand{\maxfavsim}[3]{\actvec^*_{#3}(#1, #2)}
\newcommand{\numclasses}{C}             
\newcommand{\numsimfeatures}{m}         
\newcommand{\classidx}{k}               
\newcommand{\predclassidx}{c}           
\newcommand{\decisionhead}{h}
\newcommand{\prototypedim}{\mathbb{P}}

\newcommand{\activationspace}{\reals^\numsimfeatures}
\newcommand{\outputspace}{\reals^\numclasses}
\newcommand{\argmax}{\operatorname*{argmax}}

\newcommand{\classone}{\htexttt{emperor\_\allowbreak penguin}}
\newcommand{\classtwo}{\htexttt{royal\_\allowbreak penguin}}

\newtcolorbox{running}{breakable,title={Running example},colback=cyan!5!white}

\makeatletter
\DeclareRobustCommand\ttfamily
        {\not@math@alphabet\ttfamily\mathtt
         \fontfamily\ttdefault\selectfont\hyphenchar\font=-1\relax}
\makeatother
\DeclareTextFontCommand{\htexttt}{\ttfamily\hyphenchar\font=45\relax}

\title{Formal Abductive Latent Explanations for Prototype-Based Networks}
\author {
    Jules Soria,
    Zakaria Chihani,
    Julien Girard-Satabin,\\ 
    Alban Grastien,
    Romain Xu-Darme,
    Daniela Cancila
}
\affiliations {
    Université Paris-Saclay, CEA, List, F-91120, Palaiseau, France\\
    \{jules.soria, zakaria.chihani, julien.girard2, alban.grastien, romain.xu-darme, daniela.cancila\}@cea.fr
}

\begin{document}

\maketitle

\begin{abstract}
Case-based reasoning networks are machine-learning models that make predictions based on similarity between the input and prototypical parts of training samples, called prototypes. Such models are able to explain each decision by pointing to the prototypes that contributed the most to the final outcome. As the explanation is a core part of the prediction, they are often qualified as ``interpretable by design". While promising, we show that such explanations are sometimes misleading, which hampers their usefulness in safety-critical contexts. In particular, several instances may lead to different predictions and yet have the same explanation. Drawing inspiration from the field of formal eXplainable AI (FXAI), we propose Abductive Latent Explanations (ALEs), a formalism to express sufficient conditions on the intermediate (latent) representation of the instance that imply the prediction. Our approach combines the inherent interpretability of case-based reasoning models and the guarantees provided by formal XAI. We propose a solver-free and scalable algorithm for generating ALEs based on three distinct paradigms, compare them, and present the feasibility of our approach on diverse datasets for both standard and fine-grained image classification. The associated code can be found here\footnote{\url{https://github.com/julsoria/ale}}.
\end{abstract}




\section{Introduction}
A widely adopted approach to explain neural network decisions is to analyze the
decisions of a model after its training, in a post-hoc
fashion \cite{molnar2025interpretable}. For neural networks in computer vision,
a common line of work consists in
computing the most relevant pixels 
by backpropagating gradients on the input space for a given sample
 \cite{smilkov2017smoothgrads,sundararajan2017axiomatic,choi2025unlearning,montavon2019layer}.

However, such approaches are not without flaws. They have been shown to be
sensitive to malign manipulations \cite{dombrowski2019explanations}, 
raising questions on their usefulness in a setting where a stakeholder wants to provide correct explanations
 \cite{bordt2022post}.
Moreover, some attribution methods may not correlate to the actual model behavior \cite{adebayo2018sanity,hedstroem2024fresh} --- raising questions on what they truly aim to explain --- or display irrelevant features \cite{marquessilva2024explainability}. Finally, their usefulness on actual scenarios involving human users has been challenged \cite{colin2022cannot}.

To overcome such limitations, the emerging field of \emph{formal explainable AI
(FXAI)} \cite{audemard2021explanatory,marques2022delivering,bassan2023towards,shi2020tractable,wolf2019formal}
provides a rigorous framework to characterize and build explanations. 
A recent line of work~\cite{marques2022delivering,bassan2023towards,depalma2024using,wu2023verix}
builds upon the framework of \emph{abductive reasoning}
where explanations are
defined as a subset of input features that are sufficient to justify the model
decision. In particular, it is possible to produce subset-optimal explanations
within this framework, such that removing any single feature from such
explanations can include elements that change the classifier's decision. 
FXAI provides strong guarantees on the relevance of features in the explanation, thanks to the use of automated provers that directly query the model. 
As a result, FXAI represents a good compromise between compactness and correctness --- which are deemed important characteristics of an explanation \cite{nauta2023anecdotal,miller2019explanation}.

Although FXAI is a promising approach, it
suffers from two main shortcomings:
\begin{enumerate}
    \item such approaches rely on expensive prover calls, impacting scalability on realistic computer vision tasks: such problems are generally NP-complete \cite{katz2017reluplex};
    \item abductive FXAI provides explanations at the \emph{feature-level},
    which, for typical computer vision applications, is a pixel. We argue that
    the pixel-level is not the correct level of abstractions for the human
    final user of the explanation. Pixel-level explanations
    rely on the model's perception of the problem, creating a knowledge gap
    between the human and the machine \cite{miller2019explanation}.
    On the other hand, prototypes or concepts allow generalizing facts towards higher-level reasoning \cite{lake2016building}.
\end{enumerate}


\emph{Case-based reasoning} is an orthogonal approach to FXAI to explain neural network decisions. 
Under this setting, the neural network is designed to justify its
decision by exposing examples from its dataset that are similar to the new
sample. Such approach is exemplified by prototype
learning \cite{chen2019this,rymarczyk2020protopshare,van2021interpretable,nauta2021neural,rymarczyk2022interpretable,willard2024looks, sacha2024interpretability}
or concept learning \cite{kim2018interpretability,fel2023craft,de2025v,koh2020concept, espinosa2022concept, helbling2025conceptattention}.
Prototype-based approaches have been shown to be human-interpretable through user-studies \cite{davoodi2023interpretability}.

Case-based reasoning
explanations justify the decision by exposing a certain number of
prototypes or concepts to the user. One major
drawback is that the prototypes are usually not sufficient to entail
the decision.
In the original
ProtoPNet \cite{chen2019looks}
architecture, the number of prototypes included in the explanation is fixed as an arbitrary hyperparameter. 
The resulting explanations could omit relevant prototypes,
resulting in misleading explanations.


\subsection*{Contributions} 


In this paper, we aim to bridge the gap between FXAI and prototype-based
learning. 
We show that, given a trained case-based reasoning network, we can  produce abductive
explanations that are correct not only at the pixel-level, but also in the latent space.
We consider that the produced explanations are more suitable for
humans than pixel-level ones, while being sufficient to justify the model's
behavior.

We propose a framework to describe Abductive Latent
Explanations (ALE) for prototype-based networks. Crucially, our formalism is generic as it can
be instantiated given a definition of \emph{feature extractor},
\emph{prototype} and how \emph{similarity} links prototypes to the current
sample.

We provide three paradigms for computing ALEs, in order
to circumvent the need for costly prover calls, and produce explanations that guarantee the prediction.

We evaluate our approach on numerous datasets and architectures variations to show the feasibility of building ALEs and their drawbacks.



\section{Preliminaries}
\label{sec:bg}


Our approach requires a system with the following components:

\begin{enumerate}
    \item an \textbf{image encoder} $\encoder$ 
whose goal is to map the input from an input space $\inputspace$ to a latent representation space $\latentspace$;
    \item a \textbf{prototype layer} that measures the distance between the \textit{latent} representation $\latentrep$ of the image and learned prototypes $\prototypej{j}$, $j \in \prototypeindices$, and assigns to them an \textit{activation} score;
    \item  a \textbf{decision layer} $\predictorLatent$ that
    leverages the prototype activation scores to perform instance classification.
\end{enumerate}

In the rest of our paper, we leverage the prototype and decision layers as defined in  the original ProtoPNet architecture \cite{chen2019looks} and follow-ups \cite{willard2024looks} as they represent the state-of-the-art of prototype-based models, and implementations are readily available \cite{xu2024cabrnet}. 



To give an intuition on abductive explanations in the latent space, 
we provide
a running example 
after briefly introducing the notations used in our paper.

\subsection*{Notations}
Refer to Table~\ref{tab:notations} for an overview of all notations.
The predictor function $\predictor: \inputspace \to \classset$ is defined as $\predictor = \argmax \circ \,\decisionhead \circ \actvec \circ \encoder$, where $\encoder$ maps inputs to the latent space, $\actvec$ computes prototype activations, and $\decisionhead$ produces class logits. 
The similarity matrix $\mathbf{sim}(\latentrep, \mathbf{p})$, with $\mathbf{p} = (\prototypej{j})_{j \in \prototypeindices}$, has entries $\text{sim}(\latentrepcomp{l}, \prototypej{j})$ for $l \in \latentindices$, $j \in \prototypeindices$. Explanations $\explanation$ are a subset of the latent features in $\prototypeindices\times\latentindices$.


\begin{table}[t!]
\small
\centering
\begin{tabularx}{\columnwidth}{@{}l l >{\raggedright\arraybackslash}X @{}}
\toprule
\textbf{Symbol} & \textbf{Domain} & \textbf{Description} \\
\midrule
$\classset$ & Classes & Set of $\numclasses$ classes, $\{1,\ldots,\numclasses\}$. \\
$\inputspace$ & $\mathbb{R}^{H_0 \times W_0 \times C_0}$ & Input space (e.g., images). \\
$\latentspace$ & $\mathbb{R}^{H_1 \times W_1 \times \latentcomponentdim}$ & Latent representation space. \\
$\prototypedim$ & $\mathbb{R}^\latentcomponentdim$ & Latent component space. \\
$\activationspace$ & & Prototype activation vector space. \\
$\outputspace$ & & Class logit space. \\
\midrule
$\latentindices$ & $\mathbb{N}$ & component indices, e.g., $H_1 \times W_1$. \\
$\prototypeindices$ & $\mathbb{N}$ & prototype indices, $\{1,\ldots,m\}$. \\
\midrule
$\instancevec; \inputvec$ & $\in \inputspace$ & Input instance (e.g., image). \\
$\instancelatentrep; \latentrep$ & $\in \latentspace$ & Latent representation of $\instancevec$; $\inputvec$. \\
$\latentrepcomp{l}$ & $\in \prototypedim$ & $l$-th latent component, $l \in \latentindices$. \\
$\prototypej{j}$ & $\in \prototypedim$ & $j$-th prototype, $j \in \prototypeindices$. \\
$\actvec(\latentrep)$ & $\in \activationspace$ & Activation vector. \\ & & $j$-th component is $\actveci{j}(\latentrep)$. \\ 
$\predclassidx$ & $\in \classset$ & Predicted class. $\predclassidx\!=\!\predictor(\instancevec).$\\
$m$ & $\in \mathbb{N}$ & Number of prototypes.\\
\midrule
$\explanation$ & $\prototypeindices \times \latentindices$ & Abductive Latent Explanation. \\
$\actvecexp$ & $\subseteq \activationspace$ & Constrained activation space. \\
$\maxsimval{\explanation, j}$ & $\in \mathbb{R}$ & Upper bound for $j$-th activation. \\
$\minsimval{\explanation, j}$ & $\in \mathbb{R}$ & Lower bound for $j$-th activation. \\
\midrule
$d(\cdot, \cdot)$ & $\prototypedim \times \prototypedim \to \mathbb{R}$ & Distance metric in $\prototypedim$. \\
$\sigma(\cdot)$ & $\mathbb{R} \to \mathbb{R}$ & Distance-to-similarity function. \\
$\text{sim}(\cdot, \cdot)$ & $\prototypedim \times \prototypedim \to \mathbb{R}$ & Similarity, $\text{sim}(a,b) \!=\! \sigma(d(a,b))$. \\
$\simub(\cdot, \cdot)$ & $\prototypedim \times \prototypedim \to \mathbb{R}$ & Upper bound on similarity. \\
$\simlb(\cdot, \cdot)$ & $\prototypedim \times \prototypedim \to \mathbb{R}$ & Lower bound on similarity. \\

\midrule
$\encoder$ & $\inputspace \to \latentspace$ & Maps inputs to latent space. \\
$\actvec$ & $\latentspace \to \activationspace$ & Prototype activation function:\\
& & $j\in\prototypeindices. \actveci{j}(\latentrep) \!=\!\quad \underset{l \in \latentindices}{\max}\; \text{sim}(\latentrepcomp{l}, \prototypej{j})$. \\

$\decisionhead$ & $\activationspace \to \outputspace$ & Outputs class logits. \\
$\predictor$ & $\inputspace \to \classset$ & Predictor function:\\ 
& & $\quad\predictor = \argmax \circ\, \decisionhead \circ \actvec \circ \encoder$. \\

$\predictorLatent$ & $\latentspace \to \classset$ & Latent predictor function:\\
& & $\quad\predictorLatent = \argmax \circ\, \decisionhead \circ \actvec$. \\
\bottomrule
\end{tabularx}
\caption{Summary of Notations Used.}
\label{tab:notations}
\end{table}

\subsection*{Running example}
\label{sec:running-ex}

\begin{figure*}[t!]
\centering
{
    \includegraphics[width=\textwidth]{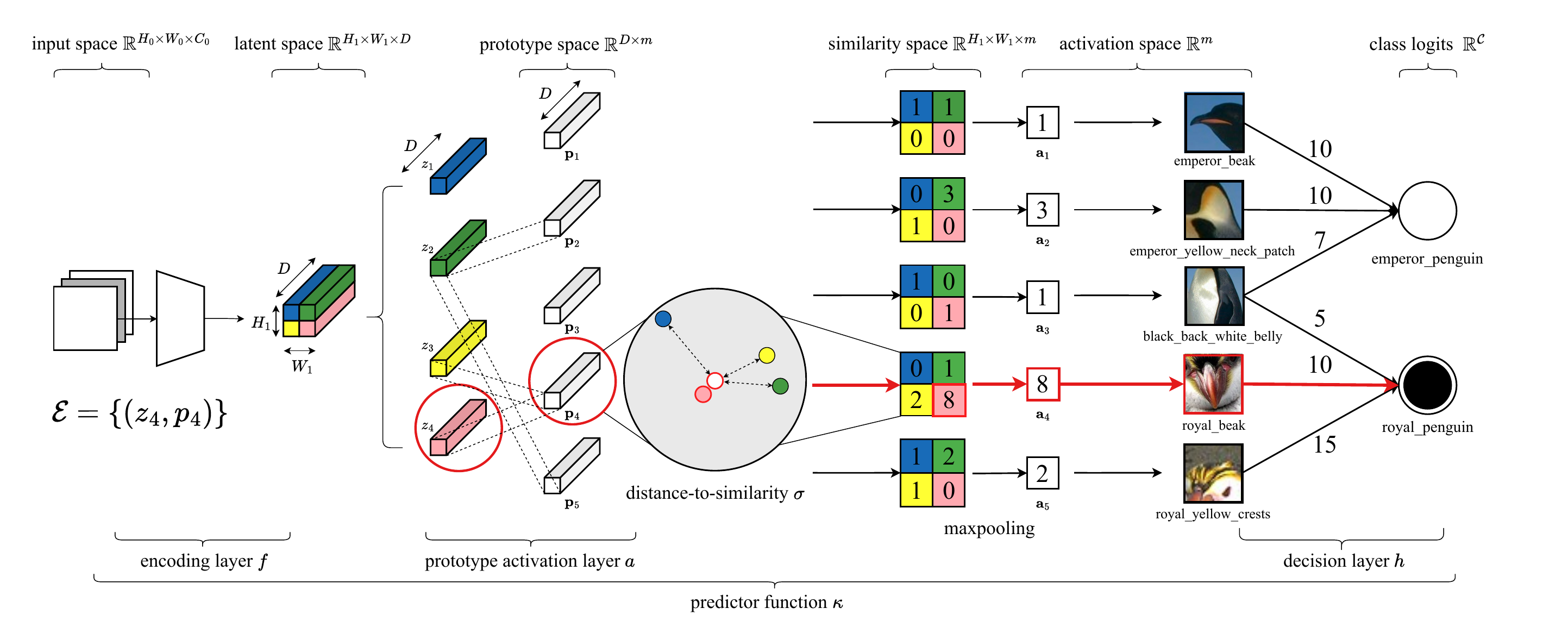}
}
\caption{Example of a top-$1$ explanation for a ProtoPNet with five prototypical parts for two classes.
 }
\label{fig:running-ex}
\end{figure*}

We consider a 2-class classification problem with \classone\ and \classtwo.
During inference, an image input $\instancevec$ first goes through an \textbf{image encoder} $\encoder$ and is represented as $\latentrep$, i.e. an object in the latent space $\latentspace$, which sums up the important concepts \emph{locally }identified. This latent representation is composed of 4 ($H_1\times W_1$) latent vectors, with width $W_1=2$ and height $H_1=2$ (Figure~\ref{fig:running-ex}).

A \emph{prototype} is a latent vector extracted from a training image that the training procedure
has identified as representative of some class.  
In our example, we assume that the training procedure computed five prototypes, akin to the following concepts:
\htexttt{emperor\_\allowbreak beak}, \htexttt{emperor\_\allowbreak yellow\_\allowbreak neck\_\allowbreak patch},
\htexttt{black\_\allowbreak back\_\allowbreak white\_\allowbreak belly}, \htexttt{royal\_\allowbreak beak},
and \htexttt{royal\_\allowbreak yellow\_\allowbreak crests} (these labels are solely used to make the example easy to follow; they are neither inputs of the training nor produced by it).
We use $\prototypej{}$ to refer to the vector of prototypes and $\prototypej{j}$ for the $j$-th prototype.

Secondly, the \textbf{prototype layer}  $\basesimilarity(\latentrep,\prototypej{})$ 
computes a similarity score between each latent vector (row) and the prototypes (column), where higher value means higher similarity:
\begin{align*}
    \basesimilarity(\latentrep,\prototypej{}) &= 
    \begin{bmatrix}
1. & 0. & 1. & 0. & 1. \\
1. & 3. & 0. & 1. & 2. \\
0. & 1. & 0. & 2. & 1. \\
0. & 0. & 1. & 8. & 0. 
\end{bmatrix}
\end{align*}
In this example,  the fourth prototype ({\htexttt{royal\_\allowbreak beak}}, fourth column) has been strongly recognized
in the bottom-right part of the image (fourth row), hence a similarity score of $8$.
Similarly, the second prototype (\htexttt{emperor\_\allowbreak yellow\_\allowbreak neck\_\allowbreak patch}) is moderately recognized, with a score of $3.$, 
while the other prototypes are absent (score 2 or lower).
For each prototype, the ProtoPNet architecture is only interested in the top similarity, and computes the column-wise maximum value of $\basesimilarity(\latentrep,\prototypej{})$ called the \emph{activation vector}
\begin{align*}
    \actvec(\latentrep) &= \begin{bmatrix}
1. & 3. & 1. & 8. & 2.
\end{bmatrix}
\end{align*}
These values mean that the network detected a clear royal beak and what appears to be an emperor yellow neck patch, but no other attribute one would expect from a penguin.

Thirdly, in the \textbf{decision layer}, the \emph{activation} vector
is fed to a fully connected linear layer with learned parameters $W$: 
\begin{align*}
    W = \begin{bmatrix}
        10. & 10. & 7. & 0.  & 0.\\
        0.  & 0.  & 5.  & 10. & 15.
        \end{bmatrix}
\end{align*}
Matrix $W$ indicates that prototypes $1$ and $2$
are typical of emperor penguins (weights $10.$ against $0.$)
and prototypes $4$ and $5$ of royal penguins (weights $10.$ and $15.$ versus $0.$),
while prototype $3$ can be detected in both classes.
The decision layer's computation returns the final class scores:
\begin{align*}
    h(\actvec(\latentrep)) &=  W\actvec(\latentrep) = \begin{bmatrix}
    47. & 115.
    \end{bmatrix} 
\end{align*}
and the classifier thus outputs the second class:
$\predictorLatent(\latentrep) = {\argmax}\;h(\actvec(\latentrep)) = 2$, i.e., $\classtwo$.

ProtoPNet returns as an explanation the $k$ prototypes with highest activation.
In the running example, if $k\!=\!1$, the explanation is $\explanation\!=\!\{(\latentrepcomp{4},\prototypej{4})\}$ which involves $\text{sim}(\latentrepcomp{4},\prototypej{4})\!=\! 8$.
We highlight that the above explanation implicitly entails
that the score of all other prototypes is $8$ or below.
At first glance, such an explanation can appear appropriate:
the bird on the image is classified as a \classtwo\,
because a royal beak (typical of royal penguins) has been observed.

However, a closer examination reveals that the explanation is wrong for some instances.
Let us consider the following counter-example:
an image having latent representation $\latentrep'$
such that the activation evaluates to
\begin{align*}
    \actvec(\latentrep') &= \begin{bmatrix}
6. & 7. & 1. & 8. & 2.
\end{bmatrix}
\end{align*}
The classification is $\predictorLatent(\latentrep') \!=\! \argmax \begin{bmatrix}
    137. & 115.
    \end{bmatrix} \!=\! 1$
while the previous explanation (``{\em because it shows a royal beak}'') is still applicable to \classone. 
This stems from the fact that, following ProtoPNet fashion, we defined the explanation to only take the \emph{top-activated prototype}, while other, less active but still important prototypes, are needed to fully justify the decision.
In other words, the explanation
is misleading or \emph{optimistic} as discussed in \cite{ignatiev2019validating}. 
Hence, the need for formal guarantees on prototypes explanations --- abductive latent-based explanations.

\section{Abductive Latent Explanations}

\label{sec:latent-abductive}
An Abductive eXplanation (AXp) for $\instancevec$ is traditionally defined as a condition on the input of a classifier
satisfied by the current instance
such that all instances that satisfy this condition yield the same output \cite{ignatiev2019abduction}. 
In other words, an AXp defines \emph{preconditions} on the inputs of a classifier, yielding a \emph{postcondition} on the classification of said classifier \cite{dijkstra1976discipline}. 
Formally, using notations presented in Table~\ref{tab:notations}, given 
an instance $\instancevec$ with prediction $\predclassidx$,
an \emph{abductive explanation} is a precondition $\sigexp(\inputvec,\instancevec)$
over input $\inputvec$ satisfied by $\instancevec$
such that:
\begin{equation*}
    \forall \inputvec \in \inputspace. \quad \sigexp(\inputvec,\instancevec) \Rightarrow (\predictor(\inputvec) = c).
\end{equation*}
In previous works \cite{marques2022delivering}, the explanation is represented as a subset $\explanation$ of input variables (features),
and the precondition $\sigexp$ simply states that the assignments of these variables
should match those of $\instancevec$:
\begin{equation*}
    \sigexp(\inputvec,\instancevec) = \bigwedge_{i \in \explanation} (\inputvec_i = \instancevec_i).
\end{equation*}

Our main contribution is the extension of the definition of AXps to an \emph{arbitrary latent space}.

In the case of image recognition,  an explanation
is then a subset of the image pixels.
While this explanation is correct (i.e., any image that includes these specific pixels
will be classified the same) 
its interpretability is questionable.



Since
 explanations are preconditions on the input space of the classifier, in the context of case-based reasoning, we can define such preconditions on the \emph{input of the latent classifier}.

\begin{definition}[Abductive Latent Explanation (ALE)]
  \label{def:latent-based-exp2}
  Given an input instance $\instancevec$ with latent representation 
  $\encoder(\instancevec)$,
  an \emph{abductive latent explanation} is a subset of latent features $\explanation$ that entails a precondition
  $\sigexp$
  over $\encoder(\inputvec)$ satisfied by $\encoder(\instancevec)$.
  
  Given an input space $\inputspace$ and
  a predictor $\predictor = \predictorLatent \circ \encoder$
  from $\inputspace$ to $\classset$,
  the explanation is \emph{correct} if the following holds:
\begin{equation*}
    \forall \inputvec \in \inputspace.
    \quad \sigexp(\encoder(\inputvec),\encoder(\instancevec))
    \Rightarrow (\predictor(\inputvec) = c).
\end{equation*}
\end{definition}

If $\explanation$ is an ALE, it is not necessarily subset-minimal. Indeed, when $\explanation \subseteq \prototypeindices\times\latentindices$ then $\prototypeindices\times\latentindices$ is an ALE. We take a particular interest in defining (and later computing) subset-minimal explanations, assuming 
that a smaller subset is more human-interpretable than the whole superset.

\begin{definition}[Subset-Minimal ALE]
\label{def:subset-minimal-ale}
An Abductive Latent Explanation $\explanation$ for an instance $\instancevec$ is \textbf{subset-minimal} if no proper subset of it is also an ALE for the same instance. 
This can be expressed formally as:
\begin{equation*}
\forall \explanation' \subsetneq \explanation, \;\; \exists \inputvec' \in \inputspace \text{ s.t. } \Big( \phi_{\explanation'}(\encoder(\inputvec'),\encoder(\instancevec)) \land (\predictor(\inputvec') \neq c) \Big).
\end{equation*}
\end{definition}

Similarly, we can formalize the implicit definition of an explanation that is used in ProtoPNet as follows:
\begin{definition}[ProtoPNet Explanation]
\label{def:protopnetxp}
  Given a set $\prototypeindices$ of prototype indices,
  a \emph{ProtoPNet explanation} $\explanation \subseteq \prototypeindices$ is a subset of indices
  that implicitly represents the precondition $\sigexp$:
  \begin{equation*}
  \sigexp(\latentrep, \instancelatentrep) \!=\!
    \left(\bigwedge_{i \in \explanation}
      \actveci{i}(\latentrep) \!=\! \actveci{i}(\instancelatentrep)
    \right)
    \!\land\!
    \left(\bigwedge_{\substack{j \not\in \explanation \\ i \in \explanation}}
      \actveci{j}(\latentrep) \!\le\! \actveci{i}(\instancelatentrep)
    \right).
  \end{equation*}
\end{definition}

Compared to the pixel-based abductive explanations,
we consider that ALEs are more interpretable
as they refer to the concepts that humans are able to manipulate.
Furthermore, compared to relevance-based explanations,
ALEs provide formal guarantees
as it is impossible to come up with misleading 
(or optimistic) 
explanations (i.e., no sample matching the explanation would be classified differently).


\section{Building abductive explanations in the latent space}\label{sec:building-abductive} 
\label{sec:build_ale}

In this section, we show how to design the precondition entailed by an ALE
that relies on the bounds of intermediate prototype activation scores.
Thus, an explanation $\explanation$ will implicitly entail a condition of the form
\begin{align*}
    \sigexp(\latentrep, \instancelatentrep) &= 
    \left(\wedge_{j=1}^m \actveci{j}(\latentrep) \in \left[\minsimval{\explanation, j}, \maxsimval{\explanation,j} 
    \right]
    \right).
\end{align*}

In this setting, recall from Section~\ref{sec:bg} that activations are computed from similarity scores, such that $\text{sim}(a,b) \!=\! \sigma(d(a,b))$ where $\sigma$ is monotonous function which \textit{increases} as the distance \textit{decreases}\footnote{In \cite{chen2019looks}, the function $\sigma(x)\!=\!\log\left(\frac{x+1}{x+\epsilon}\right)$ is used.}. Boundaries on one can be translated to boundaries on the other. Therefore, we focus next on deducing boundaries on the distance between prototypes and feature vectors, which propagate into boundaries in the activation space, then into boundaries in the class logits space.

\subsection{Domain and Spatial Constraints}

Domain-restricted AXps \cite{yu2023eliminating} refine the classical definition by constraining "unfixed" variables in the domain space.
Similarly, we note that prototypes are fixed in the latent space $\latentspace$. As such, latent space feature vectors are bounded by their distance to these prototypes. Thus, we can 
define a
precondition 
which relies on 
the \emph{relationship between a latent vector and its location \emph{w.r.t.} other prototypes}. 

A cornerstone of our approach is to analyse how giving more information about the distance between a latent component and a prototype changes the interval boundaries for the \emph{other prototype activation scores}, exploiting these spatial constraints. Furthermore, what follows in this section is true for any \emph{true} distance function, such as the $L_2$ distance.

We propose two methods to compute the interpretation (the bounds on the activation) of these explanations.
In our case,
an ALE $\explanation$ is a subset of $\prototypeindices \times \latentindices$ (where $\latentindices$ is a set of indices of components in $\latentspace$, as defined in Table~\ref{tab:notations}).

\subsubsection{Triangular Inequality}
In order to reach prototype activation score boundaries 
while ensuring stronger ``awareness"
that the prototypes and feature vectors co-exist in the same dimensional space $\prototypedim$, we propose to use the \emph{triangular inequality}, a property of distance functions.
Given 
$\mathbf{z}_l, \mathbf{p}_i, \mathbf{p}_j$~:
\begin{equation*}
    |d(\mathbf{z}_l,\mathbf{p}_j) \!-\! d(\mathbf{p}_j,\mathbf{p}_i)| \leq d(\mathbf{z}_l,\mathbf{p}_i) \leq d(\mathbf{z}_l,\mathbf{p}_j) \!+\! d(\mathbf{p}_j,\mathbf{p}_i).
\end{equation*}

By definition, the prototypes $\mathbf{p}_{1,\ldots,m}$ are an integral part of the classifier model; we have access to $d(\mathbf{p}_i,\mathbf{p}_j), \forall i,j$.

Given the distance between a singular feature vector $\latentrepcomp{l}$ and a prototype $\prototypej{j}$, the Triangular Inequality naturally provides boundaries on the similarity score $\text{sim}(\latentrepcomp{l}, \prototypej{i})$ between that feature vector $\latentrepcomp{l}$ and all other prototypes $\prototypej{i\neq j}$. 
From these similarity boundaries (for all feature vector-prototype pairs), an upper and lower bound on the prototype activation score $\actvecexp$ can be deduced.
An ALE $\explanation$ then entails $\sigexp$ with the boundaries $\{\simlb, \simub\}$ defined through:
\begin{align*}
    &\forall (l,j) \in \explanation.\, \simub(\latentrep_{l}, \prototypej{j}) \!=\! \simlb(\latentrep_{l}, \prototypej{j}) \!=\! \text{sim}(\latentrep_{l}, \prototypej{j})\\
    &\forall (l,i) \notin \explanation.\;
    \simlb(\latentrep_{l}, \prototypej{i}) \!=\! \underset{(l,j) \in \explanation}{\max}\sigma\!\left(d\left(\latentrep_{l}, \prototypej{j}\right) \!+\! d\left(\prototypej{j}, \prototypej{i}\right)\right)\\ &\quad\quad\quad\quad\;\;\;
    \simub(\latentrep_{l}, \prototypej{i}) \!=\! \underset{(l,j) \in \explanation}{\min}\sigma\!\left(|d(\latentrep_{l}, \prototypej{j}) \!-\! d\!\left(\prototypej{j}, \prototypej{i}\right)\!|\right)\!.
\end{align*}
These boundaries give $\actvecexp$ obtained by,
\begin{align*}
    \forall j\in\prototypeindices.\quad\minsimval{\explanation, j} &= \underset{l \in \latentindices}{\max}\;\simlb(\latentrepcomp{l}, \prototypej{j})\\
    \maxsimval{\explanation, j} &= \underset{l \in \latentindices}{\max}\;\simub(\latentrepcomp{l}, \prototypej{j})
    .
\end{align*}

\subsubsection{Hypersphere Intersection Approximation}
An alternative way of devising such constraints is to see the latent feature vectors as intersections between hyperspheres - where those hyperspheres have the prototypes as centers, and the distance between the vector and those prototypes as radii.
It is thus possible to deduce distance 
boundaries by \emph{overapproximating hyperspheres intersection}. Such intersection is itself a hypersphere that contains the latent feature vector.
In this spatial paradigm, at each ``iteration", i.e., for a latent feature vector, every time we update/extend the explanation by adding to it a distance between that vector and a new prototype, we refine the size of the hypersphere containing it.

One 
advantage
of this approximation is that the resulting hypersphere is, at worst, the same size (radius-wise) as the smallest of the two intersecting hyperspheres. When we use the previous approximation as one of the two hyperspheres, we have the guarantee that the next approximation will necessarily be better.
In contrast, when relying on the triangular inequality, adding a feature-prototype pair to the explanation may not help refine the boundaries concerning that feature vector.
An additional figure
which helps visualise the approximation
can be found
in the Appendix.
\begin{definition}[Hypersphere Intersection Approximation]
\label{def:hypersphere_radius}
Let $H_1$ and $H_2$ be two hyperspheres in a Euclidean space $\mathbb{R}^{\latentcomponentdim}$ with centers $\mathbf{C}_1, \mathbf{C}_2 \in \mathbb{R}^{\latentcomponentdim}$ and radii $r_1, r_2 > 0$, respectively. Assume the hyperspheres have a non-empty intersection, which requires the distance between centers $d = \|\mathbf{C}_1 - \mathbf{C}_2\|_2$ to satisfy $|r_1 - r_2| \le d \le r_1 + r_2$.
Let $H_3$ be the hypersphere approximating the intersection $H_1 \cap H_2$, constructed as described above. Then, $H_3$ has a radius $r_3$:
\[
r_3 = \frac{2}{d} \sqrt{p(p - d)(p - r_1)(p - r_2)}
\]
where $p = \frac{1}{2}(d + r_1 + r_2)$.
The center $\mathbf{C}_3$ of $H_3$ 
is:
\[
\mathbf{C}_3 = \mathbf{C}_1 + \sqrt{r_1^2 - r_3^2} \cdot \frac{\mathbf{C}_2 - \mathbf{C}_1}{d}
.\]

\end{definition}

\begin{theorem}[Containment and Minimality of the Hypersphere Intersection Approximation]
\label{thm:containment_minimality}
Let $H_1$ and $H_2$ be two hyperspheres in a Euclidean space $\mathbb{R}^{\latentcomponentdim}$ with centers $\mathbf{C}_1, \mathbf{C}_2$ and radii $r_1, r_2$, respectively. Let their surfaces intersect in a non-empty set $S_{int}$, defined as:
\[
S_{int} = \{ \mathbf{s} \in \mathbb{R}^{\latentcomponentdim} \mid \|\mathbf{s} - \mathbf{C}_1\|_2 = r_1 \text{ and } \|\mathbf{s} - \mathbf{C}_2\|_2 = r_2 \}
.\]
Let $H_3$ be the approximating hypersphere with center $\mathbf{C}_3$ and radius $r_3$ as defined in Definition~\ref{def:hypersphere_radius}.
Then:
\begin{enumerate}
    \item \textbf{(Containment)} The intersection of the surfaces is entirely contained within the approximating hypersphere $H_3$. That is, $S_{int} \subseteq H_3$.
    \item \textbf{(Minimality)} $H_3$ is the smallest possible hypersphere by radius that contains $S_{int}$. For any hypersphere $H'$ with radius $r'$ such that $S_{int} \subseteq H'$, it must be that $r_3 \le r'$.
\end{enumerate}
\end{theorem}

\begin{proof}
Found in the Appendix.
\end{proof}


With this paradigm, an ALE $\explanation$ entails $\sigexp$ with $\actvecexp$ 
:
\begin{align*}
    \forall (l,j) &\in \explanation.\\ 
    &\simub(\latentrep_{l}, \prototypej{j}) = \simlb(\latentrep_{l}, \prototypej{j}) = \text{sim}(\latentrep_{l}, \prototypej{j})\\
    \forall (l,j) &\notin \explanation.\\ 
    &\simub(\latentrep_{l}, \prototypej{j}) = \sigma\!\left(d\left(C_{\latentrepcomp{l}}, \prototypej{j}\right) - r_{\latentrepcomp{l}}\right)\\  
    &\simlb(\latentrep_{l}, \prototypej{j}) = \sigma\!\left(d\left(C_{\latentrepcomp{l}}, \prototypej{j}\right) + r_{\latentrepcomp{l}}\right)
\end{align*}
with $C_{\latentrepcomp{l}}\text{ and } r_{\latentrepcomp{l}}$ the center and radius, respectively, of the approximated hypersphere containing $\latentrepcomp{l}$.

This approximation is possible because, by construction, all considered hyperspheres \textbf{do} intersect in at least one point, the latent feature vector we attempt to approximate.

\subsubsection{Top-k explanations}
This paradigm is the one used implicitly by the original explanation \cite{chen2019looks} provided by ProtoPNet. It traverses the prototype activation scores in decreasing order, with the added knowledge that, for that prototype, the similarity scores with the other latent space feature vectors are lesser than the activation score (result of the \texttt{max} function).
In that scenario, we obtain $\actvecexp$ from $\explanation$ by:
\begin{align*}
    \forall j \in \explanation&.\quad
    \minsimval{\explanation, j} = \maxsimval{\explanation, j} = \actveci{j}(\instancelatentrep)\\
    \forall j \notin \explanation&.\quad
    \minsimval{\explanation, j} = 0 \;\;\text{and}\;\;
    \maxsimval{\explanation, j} = \underset{i \in \explanation}{\min}\;\actveci{i}(\instancelatentrep).
\end{align*}

Notice that, in the top-$k$ explanation, the explanation is based on the activation space directly. The latent feature vector relevant for that activation is implicitly included in $\explanation$ so it behaves as if $\|\latentindices\|\!=\!1$.

\subsection{Constructing Abductive Latent Explanations}

Given a candidate explanation $\explanation \subseteq \prototypeindices\times\latentindices$, the constrained \textbf{prototype activation} space is represented by
\begin{displaymath}
\actvecexp= \{[0,\;\maxsimval{\explanation, j}]\}_{j\in\prototypeindices}
.\end{displaymath}

Following our preliminary work \cite{soria2025towards}, we introduce the notion of constructing an abductive latent explanation $\explanation$ in the prototype activation space by defining preconditions on the activation vector $\actvec \in \activationspace$. These preconditions are derived from $\explanation$ and $\instactval{}$. Crucially, they must guarantee that any activation vector $\actvec$ satisfying these conditions results in the same predicted class $\predclassidx$:
$$ \forall \actvec \in \actvecexp \subseteq \activationspace : \actvec \models\ \sigexp, \quad \argmax_{\classidx \in \classset} \logit{\classidx}(\actvec) = \predclassidx.$$
These preconditions effectively define a constrained region -  or set of constraints $\actvecexp$ within the activation space $\activationspace$ - such that $\instactval{}$ satisfies these constraints, and all vectors within this region also yield the prediction $\predclassidx$.

\begin{assumption}[Linear Logit Difference] \label{ass:linearity}
For any two classes $\classidx, \predclassidx$, the difference between their logit functions is linear in the prototype activation vector $\actvec$:
$$ \logit{\classidx}(\actvec) - \logit{\predclassidx}(\actvec) = \sum_{j=1}^{m} (w_{j\classidx} - w_{j\predclassidx})\;\actvec_j + (b_{\classidx} - b_{\predclassidx}) $$
for some weights $w_{j\classidx}, w_{j\predclassidx}$ and biases $b_{\classidx}, b_{\predclassidx}$.
\end{assumption}

\begin{definition}[Maximally Class-Favoring Element within $\actvecexp$] \label{def:max_favor_constr}
For a given explanation $\explanation$, predicted class $\predclassidx$, and 
class $\classidx \neq \predclassidx$, the \emph{maximally $(\classidx/\predclassidx)$-favoring \textbf{prototype activation} vector within $\actvecexp$} is denoted by $\maxfavsim{\classidx}{\predclassidx}{\explanation}$ and satisfies:
$$\forall \actvec \in \actvecexp.\quad \logit{\classidx}(\maxfavsim{\classidx}{\predclassidx}{\explanation}) - \logit{\predclassidx}(\maxfavsim{\classidx}{\predclassidx}{\explanation}) \ge \logit{\classidx}(\actvec) - \logit{\predclassidx}(\actvec).$$

Under Assumption \ref{ass:linearity}, its components $(\maxfavsim{\classidx}{\predclassidx}{\explanation})_j$ are constructed as:
    $$ (\maxfavsim{\classidx}{\predclassidx}{\explanation})_j =
       \begin{cases}
           (\maxsimval{\explanation})_j & \text{if } w_{j\classidx} \geq w_{j\predclassidx} \\
           (\minsimval{\explanation})_j & \text{if } w_{j\classidx} < w_{j\predclassidx}
       \end{cases}.
    $$
\end{definition}

Intuitively, this element \emph{satisfies} the condition expressed by the explanation $\explanation$, and prevents the most the classifier from reaching its initial prediction $\predclassidx$ (in favor of class $\classidx$).

\begin{definition}[Class-wise Prediction Domination within $\actvecexp$] \label{def:domination_constr}
For explanation $\explanation$, predicted class $\predclassidx$, and alternative class $\classidx \neq \predclassidx$, we say $\predclassidx$ \emph{dominates} $\classidx$ within $\simspaceconstr{\explanation}$, denoted $\psi_{\explanation}(\classidx, \predclassidx)$, if the logit of $\predclassidx$ is greater than or equal to the logit of $\classidx$ even for the maximally $(\classidx/\predclassidx)$-favoring vector:
$$ \psi_{\explanation}(\classidx, \predclassidx) \iff \logit{\predclassidx}(\maxfavsim{\classidx}{\predclassidx}{\explanation}) \geq \logit{\classidx}(\maxfavsim{\classidx}{\predclassidx}{\explanation}) .$$
\end{definition}

\begin{definition}[Total Prediction Domination within $\actvecexp$ (Explanation Verification)] \label{def:verification_constr}
A candidate explanation $\explanation$ is considered \emph{verified} if class $\predclassidx$ dominates all other classes $\classidx \neq \predclassidx$ within the constrained space $\simspaceconstr{\explanation}$:
$$ \text{Verify}(\explanation) \iff \forall \classidx \in \{1, \dots, \numclasses\}\setminus\{\predclassidx\}, \quad \psi_{\explanation}(\classidx, \predclassidx) .$$
\end{definition}

 In Definition~\ref{def:domination_constr} we say that a class (different from the predicted class) is \emph{verified} if its \emph{Maximally Class-Favoring Element} does not entail a different classification. In Definition~\ref{def:verification_constr} we say that the explanation \emph{verifies} the prediction if \textbf{all} classes are \emph{verified}.  


If an explanation $\explanation$ is verified according to Definition \ref{def:verification_constr}, it is an abductive explanation as presented in Definition \ref{def:latent-based-exp2}.
\begin{theorem}[Verified Explanation Sufficiency] \label{thm:verified_is_sufficient}
Let $\explanation$ be a candidate explanation for the decision $(\instancevec, \predclassidx)$. If $\explanation$ is verified according to Definition \ref{def:verification_constr} (i.e., $\text{Verify}(\explanation)$ is true), then $\explanation$ is a valid abductive explanation according to Definition \ref{def:latent-based-exp2}.
\end{theorem}
\begin{proof}
Found in the Appendix.
\end{proof}

\subsection{Algorithms}
We present two algorithms for generating ALEs of model predictions. Pseudo-code for both algorithms can be found in the Appendix.
Algorithm 1
relies on the top-$k$ paradigm, where we iteratively add the next highest activated prototype to the running explanation, until the candidate is verified. \textsc{NextPrototype} selects a prototype from $\mathcal{A}$ the ordered list of prototypes by activations. Due to the fixed order of elements, the produced explanation is \textbf{sufficient} and \textbf{cardinality-minimal}.

Algorithm 2
relies on spatial constraints, as defined earlier in this section. We first go through a \textit{forward pass} where we iteratively add pairs of latent components and prototypes. From this candidate explanation, we deduce bounds on the activation space, then verify whether the running explanation's bounds guarantee the prediction. At the end of this stage, the candidate explanation is \textbf{sufficient}. Carefully, we add a \textit{backward pass} which downsizes the candidate explanation, i.e., iteratively attempts to remove pairs, keeping only those that lead the resulting pruned explanation to not be verified anymore,
until no pair can be removed --- this provides the \textbf{subset-minimality} property. While the search is \textit{quadratic} in the explanation's length, this complexity remains polynomial. This is a significant advantage over traditional formal methods (e.g., SMT/MILP-based) for pixel-wise explanations, which are \textit{NP-hard}.

Most importantly, these two algorithms do not rely on calling any external solvers.





\begin{table*}[htbp!]
    \centering
    \small
    \setlength{\tabcolsep}{3.5pt}
    \begin{tabular}{l|c|ccc||c|ccc}
        \toprule
        Dataset & \textit{Accuracy} & Triangle & Hypersphere & top-$k$ & $H_1\!\times\! W_1$ & top-$k$ (adj.) \\
        \midrule
        & & \multicolumn{3}{c||}{Avg Total / Avg Correct / Avg Incorrect} &  & \\
        \midrule
        CIFAR-10 & 0.83 & \textbf{8.7} / \textbf{6.6} / \textbf{19.4} & 20.2 / 8.9 / 28.8 & 41.4 / 36.9 / 61.9 & $1\times1$ & 41.4 / 36.9 / 61.9 \\
        CIFAR-100$\dagger$ & 0.62 & \textbf{323.2} /\textbf{ 276.7} / \textbf{394.3} & 672.9 / 574.4 / 820.2 & 896.6 / 867.6 / 940.8 & $1\times1$ & 672.9 / 867.6 / 940.8 \\
        MNIST$\dagger$ & 0.98 & \textbf{6.2} /\textbf{ 6.2} / - & 675 / 675 / - & 8.8 / 8.8 / - & $4\times4$ & 141 / 141 / - \\
        Oxford Flowers$\dagger$ & 0.72 & 546.9 / 394.8 / 973.5 & & \textbf{287.7} / \textbf{193.6} / \textbf{525.5} & $4\times4$ & 4602 / 3098 / 8408 \\
        Oxford Pet$\dagger$ & 0.82 & 3755.3 / 748.9 / 18130 & & \textbf{77.7} / \textbf{67.9} / \textbf{122.8} & $7\times7$ & 3805 / 3328 / 6016 \\
        Stanford Cars$\dagger$ & 0.90 & 5072.3 / 992.1 / 31633.6 & & \textbf{24.9} / \textbf{12.3} / \textbf{140.6} & $7\times7$ & 1219 / 600 / 6890 \\
        CUB200$\dagger$ & 0.84 & 10653.4 / 670.9 / 98000 & & \textbf{239.3} / \textbf{217.0} / \textbf{352.0} & $7\times7$ & 11725 / 10632 / 17251 \\
        \bottomrule
    \end{tabular}
    \caption{Summary of
    Average Explanation Sizes, 
    Correct/Incorrect Explanation Sizes, 
    for each dataset. 
    Accuracy on the datasets is given only for reference. The latent space dimensions are also referenced to show how the top-$k$ explanation size is 'adjusted' to fairly compare with spatial constraints ALEs. 
    $\dagger$ indicates that for that dataset, 
    the spatial constraints ALEs are computed on five randomly selected images per class. For MNIST, since no incorrect prediction was sampled, there are no obtained results on incorrect explanations sizes.
    For some higher resolution datasets, generating Hypersphere Intersection Approximation ALEs proved to be too computation intensive, and exceeded the two-day timeout. Best (smallest) results indicated in bold.}
    \label{tab:dataset_summary}
\end{table*}


\section{Experimental Study}
\label{sec:experimental_study}
\subsubsection{Computational Resources}
Our experiments were conducted on a SLURM-managed computing cluster, primarily utilizing GPU-equipped nodes for computationally intensive tasks, with each node featuring 48 cores at 2.6GHz, 187GB of RAM, and four NVIDIA V100 32GB GPUs.

\subsubsection{Methodology}
Our study leverages the \underline{Ca}se-\underline{B}ased \underline{R}easoning \underline{Net}work (CaBRNet) framework \cite{xu2024cabrnet} to train different models. 
We trained one ProtoPNet model per dataset considered. For the backbone, depending on the dataset resolution and number of classes, we used either a VGG \cite{simonyan2014very}, a ResNet \cite{he2016deep}, or a WideResNet \cite{zagoruyko2016wide}, using either the `default' ImageNet \cite{ILSVRC15} pretrained weights, or the iNaturalist \cite{su2021semi} pretrained weights. We parameterized our models such that
10 prototypes per class were initially allocated (as is the standard for case-based reasoning training). 

Existing Python libraries are known to exhibit non-trivial errors when computing distances.%
\footnote{cf. issue \url{https://github.com/pytorch/pytorch/issues/57690}}
This leads us to (sometimes) over-approximate bounds, which may also affect the cardinality of the explanations.

\subsubsection{Datasets}
In order to compute \emph{subset-minimal} ALEs, we trained several ProtoPNet models on a variety of datasets, including CUB200 \cite{wah2011caltech}, Stanford Cars \cite{krause20133d}, Oxford Pet \cite{parkhi12a}, Oxford Flowers \cite{Nilsback08}, CIFAR-10 and CIFAR-100 \cite{krizhevsky2009learning}, and MNIST \cite{lecun1998mnist}. These datasets were chosen for their diversity and relevance to the task at hand.
Several of these datasets
exhibit a hierarchical structure that allows them to be used for both fine-grained and coarse-grained classification tasks. 


When possible, we used the standard train/test splits provided by the dataset creators and followed common preprocessing procedures, such as resizing, cropping, and data augmentation techniques, described in \cite{chen2019looks}. Training models on these datasets was facilitated by their implementation in the \texttt{torchvision} library \cite{torchvision2016}.



\subsubsection{Results}\label{sec:discussion}
Table~\ref{tab:dataset_summary} summarizes the average sizes of explanations on all datasets, broken down by correct and incorrect classifications, for each paradigm.

When looking at the computations for top-$k$, we show
that 
the $10$ most activated prototypes are usually not enough to guarantee the decision.
On all datasets except MNIST, explanations on average require more than 10 similarity scores to be sufficient to justify the classification.
This result is significant as it entails that, for the models used, nearly all original ProtoPNet  explanations (according to Definition~\ref{def:protopnetxp}) are misleading, or optimistic.

Furthermore, we observe in Table~\ref{tab:dataset_summary} that, for samples that are incorrectly classified
the average explanation size is much higher. 
This would mean that information about a sample and the data distribution can be extracted from the ALEs themselves. This wide disparity is also linked to the computation time to generate an ALE --- as the candidate explanation grows during the forward pass, the time it takes to append an additional latent vector-prototype pair grows too. 
The observed phenomenon is exacerbated for spatial constraints ALE; for the CUB200, the Oxford Pet and Oxford Flowers datasets, incorrect predictions require an ALE that includes all possible pairs (i.e., $\explanation\!=\!\prototypeindices\times\latentindices$
). One interpretation of such result is that ALEs are relevant to understand the prediction only when it is correct. This finding corroborates with related work \cite{wu2024better} which uses the explanation size as proxy for out-of-distribution detection.

Another observation from our experiments is that, for small resolution datasets, the spatial constraints ALEs (for correct predictions) require \emph{less} pairs than top-$k$ ALEs, even before accounting for the relevant adjustments (multiplying by the number of latent components). We can deduce that, in certain contexts at least, formally proving the classification of an instance using our developed paradigms leads to more compact explanations.

\subsubsection{Discussion}
We show that producing subset-minimal and sufficient explanations leads to explanations of large size,  and therefore arguably not human-interpretable. This is not a limitation of our method but rather a key finding: current prototype-based networks need a lot of components to yield explanations with formal guarantees. This discovery nuances their previous claim of interpretability.



\section{Conclusion} 

In this work, we introduced Abductive Latent Explanations (ALE), a novel framework that formalizes explanations for prototype-based networks as rigorous abductive inferences within the latent space.
We proposed a solver-free algorithm for generating ALEs that drastically reduces computation time. Our analysis, enabled by ALEs, reveals that common prototype-based explanations can be misleading, raising concerns about their reliability in high-stakes decisions. Furthermore, we investigated the relationship between prediction correctness and ALE size, with findings showing that larger ALEs often correlate with incorrect predictions, suggesting ALE size as a potential proxy for model uncertainty.


\subsubsection{Limitations and Future Works}
A current limitation of our approach is the size of the obtained explanations. Explanations in the order of magnitude of thousands for a single instance can arguably be considered still too big to be actionable and interpretable by humans. Furthermore, prototypes are currently not directly linked to human concepts. We consider bridging symbolic reasoning and ALEs a worthwile and interesting research direction.
Finally, exploring training methods which lead to a more interpretable latent space, such as well separated class prototypes \cite{chen2019looks} and meaningful encodings \cite{chen2020concept} could help in generating more compact and valuable ALEs.
A direct extension of our work consists on taking further inspirations from Abductive Explanations to generate Constrastive Explanations. Extending ALEs to other prototypes modalities is also a potential future work. ALEs could also be useful to identify irrelevant latent components, opening new perspective on targeted pruning for neural networks and their explanations.



\section*{Acknowledgements}
  This publication was made possible by the use of the FactoryIA supercomputer,
  financially supported by the Ile-De-France Regional Council. This work was supported by the SAIF project, funded by the “France 2030” government investment plan managed by the French National Research Agency, under the reference ANR-23-PEIA-0006





\bibliography{aaai2026}


\appendix
\section{Technical Appendix}
This appendix presents: 
\begin{itemize}
    \item the two pseudo-codes for the algorithms presented,
    \item the proofs of the two theorems (Minimality of the Hypersphere Intersection Approximation) and (Verified Explanation Sufficiency),
\end{itemize}  


\subsection{Pseudo-codes}
\begin{algorithm}[htb!]
\caption{Generating top-$k$ ALE}
\label{alg:ale}
\textbf{Input}: Instance vector $\instancevec$, predicted class index $\predclassidx$\\
\textbf{Output}: Explanation $\explanation$
\begin{algorithmic}[1]
\STATE $\explanation = \emptyset$
\STATE UnverifiedClasses $= \classset \setminus \{\predclassidx\}$
\STATE $\mathcal{A} \gets \textsc{Sort}(\actvec(\instancevec))$
\WHILE{UnverifiedClasses $\neq \emptyset$}
    \STATE $j = \textsc{NextPrototype}(\explanation, \mathcal{A})$
    \STATE $\explanation \gets \explanation \cup j$
    \FOR{$\classidx \in$ UnverifiedClasses}
        \STATE $\actvec' \gets \maxfavsim{k}{c}{\explanation}$
        \IF{$h_{\predclassidx}(\actvec') > h_{\classidx}(\text{CEx})$}
            \STATE UnverifiedClasses $\gets$ UnverifiedClasses $\setminus \{\classidx\}$
        \ENDIF
    \ENDFOR
\ENDWHILE
\STATE \textbf{return} $\explanation$
\end{algorithmic}
\end{algorithm}


\begin{algorithm}[tb!]
\caption{Generating Spatial-Reasoning ALE}
\label{alg:spatial-ale}
\textbf{Input}: Instance vector $v$, predicted class $c$\\
\textbf{Parameter}: Paradigm 
\\
\textbf{Output}: Explanation $\explanation$
\begin{algorithmic}[1]
\STATE $\explanation = \emptyset$
\STATE $\explanation = \textsc{ExpInit}(v, c)$
\STATE UnverifiedClasses $= C \setminus \{c\}$
\WHILE{UnverifiedClasses $\neq \emptyset$}
    \STATE $(l, j) = \textsc{NextPair}(\explanation)$
    \STATE $d = \textsc{ComputeDistance}(v, l, j)$
    \STATE $\explanation \gets \explanation \cup \{(l, j)\}$
    \STATE $(\text{lb}, \text{ub}) = \textsc{GenerateBounds}(\explanation, \text{paradigm})$
    \STATE UnverifiedClasses $= \textsc{VerifExp}(\text{lb}, \text{ub})$
\ENDWHILE
\STATE MarkedPairs $= \emptyset$
\WHILE{$\explanation \neq \emptyset$}
    \IF{$|\explanation| = |\text{MarkedPairs}|$}
        \STATE \textbf{break}
    \ENDIF
    \STATE $(l, j) = \textsc{RemoveUnmarkedPair}(\explanation, \text{MarkedPairs})$
    \STATE $\explanation \gets \explanation \setminus \{(l, j)\}$
    \STATE $(\text{lb}, \text{ub}) = \textsc{GenerateBounds}(\explanation, \text{paradigm})$
    \STATE UnverifiedClasses $= \textsc{VerifExp}(\text{lb}, \text{ub})$
    \IF{UnverifiedClasses $\neq \emptyset$}
        \STATE MarkedPairs $\gets \text{MarkedPairs} \cup \{(l, j)\}$
        \STATE $\explanation \gets \explanation \cup \{(l, j)\}$
    \ENDIF
\ENDWHILE
\STATE \textbf{return} $\explanation$
\end{algorithmic}
\end{algorithm}

For Algorithm~\ref{alg:spatial-ale}, the auxiliary function \textsc{NextPair} used in the forward pass returns the closest latent vector-prototype pair not already included in the explanation. Another version we use in our experiments tries to efficiently assign the pairs to the explanation by not allowing one latent vector to be ``too much" explained, leaving other vectors behind. We refer to this as the Round-Robin assignment

In the backward pass, we ``mark" pairs that have been proven to be necessary for the explanation, i.e., removing them invalidates the explanation. We do this because, 
for two candidate explanations $\explanation_1, \explanation_2$ such that $\explanation_1 \subseteq \explanation_2$, and for pair $(l, j) \in \explanation_1$, for any paradigm:
\begin{align*}
    \text{lb}_i, \text{ub}_i = \textsc{GenerateBounds}(\explanation_i\setminus(l,j),\; \text{paradigm})\\
    \neg\textsc{VerifyExp}(\text{lb}_2,\text{ub}_2) \implies \neg\textsc{VerifyExp}(\text{lb}_1,\text{ub}_1)
\end{align*}
This means that, if $(l,j)$ is necessary for $\explanation_2$ it will be necessary for any explanation contained by it. The intuition behin function \textsc{VerifyExp} is visible in lines 7-11 of Algorithm~\ref{alg:ale}.

\subsection{Proofs}
\begin{theorem}[Containment and Minimality of the Hypersphere Intersection Approximation] 
\end{theorem}
\textbf{Theorem 1}(Containment and Minimality of the Hypersphere Intersection Approximation)

\begin{proof}
The proof of Theorem~\ref{thm:containment_minimality} proceeds by first identifying the geometric shape of the intersection $S_{int}$ and constructing the smallest hypersphere that contains it. Finally, we will show that this constructed hypersphere is identical to $H_3$.

\paragraph{Part 1: Geometric Characterization of the Intersection}
Let $\mathbf{s}$ be an arbitrary point in the intersection $S_{int}$. By definition, it satisfies $\|\mathbf{s} - \mathbf{C}_1\|_2 = r_1$ and $\|\mathbf{s} - \mathbf{C}_2\|_2 = r_2$. Let $d = \|\mathbf{C}_1 - \mathbf{C}_2\|_2$.

The set of all such points $S_{int}$ forms a $(\latentcomponentdim-1)$-dimensional sphere. Let us find its center, which we will call $\mathbf{C}_{int}$, and its radius, which we will call $r_{int}$. The center $\mathbf{C}_{int}$ must lie on the line passing through $\mathbf{C}_1$ and $\mathbf{C}_2$. Let $\mathbf{C}_{int}$ be the orthogonal projection of any point $\mathbf{s} \in S_{int}$ onto this line. By the Pythagorean theorem applied to the right-angled triangles $\triangle \mathbf{s}\mathbf{C}_{int}\mathbf{C}_1$ and $\triangle \mathbf{s}\mathbf{C}_{int}\mathbf{C}_2$:
\begin{align*}
    \|\mathbf{s} - \mathbf{C}_{int}\|_2^2 + \|\mathbf{C}_{int} - \mathbf{C}_1\|_2^2 &= r_1^2 \\
    \|\mathbf{s} - \mathbf{C}_{int}\|_2^2 + \|\mathbf{C}_{int} - \mathbf{C}_2\|_2^2 &= r_2^2
\end{align*}
Subtracting the second equation from the first reveals that the position of $\mathbf{C}_{int}$ is independent of the choice of $\mathbf{s}$:
\[
\|\mathbf{C}_{int} - \mathbf{C}_1\|_2^2 - \|\mathbf{C}_{int} - \mathbf{C}_2\|_2^2 = r_1^2 - r_2^2
\]
The radius $r_{int}$ is the constant distance $\|\mathbf{s} - \mathbf{C}_{int}\|_2$ for any $\mathbf{s} \in S_{int}$. This distance is the height of the triangle $\triangle \mathbf{s}\mathbf{C}_1\mathbf{C}_2$ relative to the base of length $d$. The area of this triangle can be expressed using Heron's formula or using base and height:
\[
\text{Area} = \sqrt{p(p-d)(p-r_1)(p-r_2)} = \frac{1}{2} d \cdot r_{int}
\]
where $p = \frac{1}{2}(d+r_1+r_2)$. From this, we derive the radius of the intersection sphere:
\[
r_{int} = \frac{2}{d} \sqrt{p(p-d)(p-r_1)(p-r_2)}
\]
Thus, $S_{int}$ is a $(\latentcomponentdim-1)$-sphere with a well-defined center $\mathbf{C}_{int}$ and radius $r_{int}$.

\paragraph{Part 2: The Smallest Containing Hypersphere}
Let $H_{approx}$ be a $\latentcomponentdim$-dimensional ball with center $\mathbf{C}_{int}$ and radius $r_{int}$. By construction, for any point $\mathbf{s} \in S_{int}$, we have $\|\mathbf{s} - \mathbf{C}_{int}\|_2 = r_{int}$. This means every point in $S_{int}$ lies on the surface of $H_{approx}$, so $S_{int} \subseteq H_{approx}$.

We now show that $H_{approx}$ is the smallest such containing ball. Let $H'$ be any other hypersphere with center $\mathbf{C'}$ and radius $r'$ that contains $S_{int}$. Since $S_{int}$ is a sphere of radius $r_{int}$, we can choose two antipodal points $\mathbf{s}_a$ and $\mathbf{s}_b$ on it, such that $\|\mathbf{s}_a - \mathbf{s}_b\|_2 = 2r_{int}$.
Since $H'$ contains $S_{int}$, it must contain $\mathbf{s}_a$ and $\mathbf{s}_b$, so $\|\mathbf{s}_a - \mathbf{C'}\|_2 \le r'$ and $\|\mathbf{s}_b - \mathbf{C'}\|_2 \le r'$. By the triangle inequality:
\[
2r_{int} \!=\! \|\mathbf{s}_a - \mathbf{s}_b\|_2 \le \|\mathbf{s}_a - \mathbf{C'}\|_2 + \|\mathbf{C'} - \mathbf{s}_b\|_2 \le r' + r' \!=\! 2r'
.\]
This implies $r_{int} \le r'$. The radius $r_{int}$ is thus the minimum possible radius for any hypersphere containing $S_{int}$.

\paragraph{Conclusion}
We have shown that the intersection $S_{int}$ is contained in a unique smallest hypersphere, which has radius $r_{int}$ and a center $\mathbf{C}_{int}$ whose distance from $\mathbf{C}_1$ is found via $\|\mathbf{C}_{int} - \mathbf{C}_1\|_2^2 = r_1^2 - r_{int}^2$.

By inspecting the formulas presented in Definition~\ref{def:hypersphere_radius}, we see that the radius $r_3$ and center $\mathbf{C}_3$ of the hypersphere $H_3$ are defined to be identical to our derived radius $r_{int}$ and center $\mathbf{C}_{int}$.
Therefore, $H_3$ is precisely the smallest-radius hypersphere that contains the intersection $S_{int}$, proving both containment and minimality.
\end{proof}

\begin{theorem}[Verified Explanation Sufficiency]
\end{theorem}
\begin{proof}
Assume the explanation $\explanation$ is verified, meaning $\text{Verify}(\explanation)$ holds. By Definition \ref{def:verification_constr}, this implies that for all classes $\classidx \neq \predclassidx$, the condition $\psi_{\explanation}(\classidx, \predclassidx)$ holds. By Definition \ref{def:domination_constr}, this means:
\begin{equation*} \label{eq:proof_domination_holds}
\forall \classidx \neq \predclassidx: \quad \logit{\predclassidx}(\maxfavsim{\classidx}{\predclassidx}{\explanation}) \geq \logit{\classidx}(\maxfavsim{\classidx}{\predclassidx}{\explanation})
\end{equation*}
where $\maxfavsim{\classidx}{\predclassidx}{\explanation} = \operatorname*{argmax}_{\simvec' \in \simspaceconstr{\explanation}} ( \logit{\classidx}(\simvec') - \logit{\predclassidx}(\simvec') )$.

We want to show that $\explanation$ satisfies Definition \ref{def:latent-based-exp2}. This requires showing that for any latent representation $\latentrep$ (corresponding to a similarity vector $\simvec$) such that its components satisfy the constraints imposed by $\explanation$ (i.e., $\simvec \in \simspaceconstr{\explanation}$), the classification is $\predclassidx$. That is, we need to show:
$$ \forall \simvec \in \simspaceconstr{\explanation}: \quad \predictorLatent(\latentrep) = \predclassidx $$
Since $\predictorLatent(\latentrep) = \operatorname*{argmax}_{\classidx'} \logit{\classidx'}(\simvec)$, this is equivalent to showing:
$$ \forall \simvec \in \simspaceconstr{\explanation}, \quad \forall \classidx \neq \predclassidx: \quad \logit{\predclassidx}(\simvec) \geq \logit{\classidx}(\simvec) $$

Let $\simvec$ be an arbitrary similarity vector in the constrained space $\simspaceconstr{\explanation}$.
Let $\classidx \neq \predclassidx$ be an arbitrary alternative class.

By the definition of $\maxfavsim{\classidx}{\predclassidx}{\explanation}$ as the maximizer of $(\logit{\classidx} - \logit{\predclassidx})$ within $\simspaceconstr{\explanation}$, we know that for our chosen $\simvec \in \simspaceconstr{\explanation}$:
$$ \logit{\classidx}(\maxfavsim{\classidx}{\predclassidx}{\explanation}) - \logit{\predclassidx}(\maxfavsim{\classidx}{\predclassidx}{\explanation}) \geq \logit{\classidx}(\simvec) - \logit{\predclassidx}(\simvec) $$

From our initial assumption that $\explanation$ is verified, we know from \eqref{eq:proof_domination_holds} that:
$$ \logit{\predclassidx}(\maxfavsim{\classidx}{\predclassidx}{\explanation}) \geq \logit{\classidx}(\maxfavsim{\classidx}{\predclassidx}{\explanation}) $$
This can be rewritten as:
$$ 0 \geq \logit{\classidx}(\maxfavsim{\classidx}{\predclassidx}{\explanation}) - \logit{\predclassidx}(\maxfavsim{\classidx}{\predclassidx}{\explanation}) $$

Combining these two inequalities, we have:
$$ 0 \geq \logit{\classidx}(\maxfavsim{\classidx}{\predclassidx}{\explanation}) - \logit{\predclassidx}(\maxfavsim{\classidx}{\predclassidx}{\explanation}) \geq \logit{\classidx}(\simvec) - \logit{\predclassidx}(\simvec) $$
This directly implies:
$$ 0 \geq \logit{\classidx}(\simvec) - \logit{\predclassidx}(\simvec) $$
Which rearranges to:
$$ \logit{\predclassidx}(\simvec) \geq \logit{\classidx}(\simvec) $$

Since $\simvec \in \simspaceconstr{\explanation}$ was arbitrary and $\classidx \neq \predclassidx$ was arbitrary, we have shown that $\logit{\predclassidx}(\simvec) \geq \logit{\classidx}(\simvec)$ for all $\simvec \in \simspaceconstr{\explanation}$ and for all $\classidx \neq \predclassidx$.
Therefore, for any $\simvec \in \simspaceconstr{\explanation}$, $\predclassidx = \operatorname*{argmax}_{\classidx'} \logit{\classidx'}(\simvec)$, which means $\predictorLatent(\latentrep) = \predclassidx$.

This fulfills the condition required by Definition \ref{def:latent-based-exp2}. Thus, if $\explanation$ is verified via Total Prediction Domination within $S_{\explanation}$, it is a sufficient abductive explanation.
\end{proof}

\subsection{Additional Figure}
\begin{figure}[htbp!]
\centering
\includegraphics[width=0.4\textwidth]{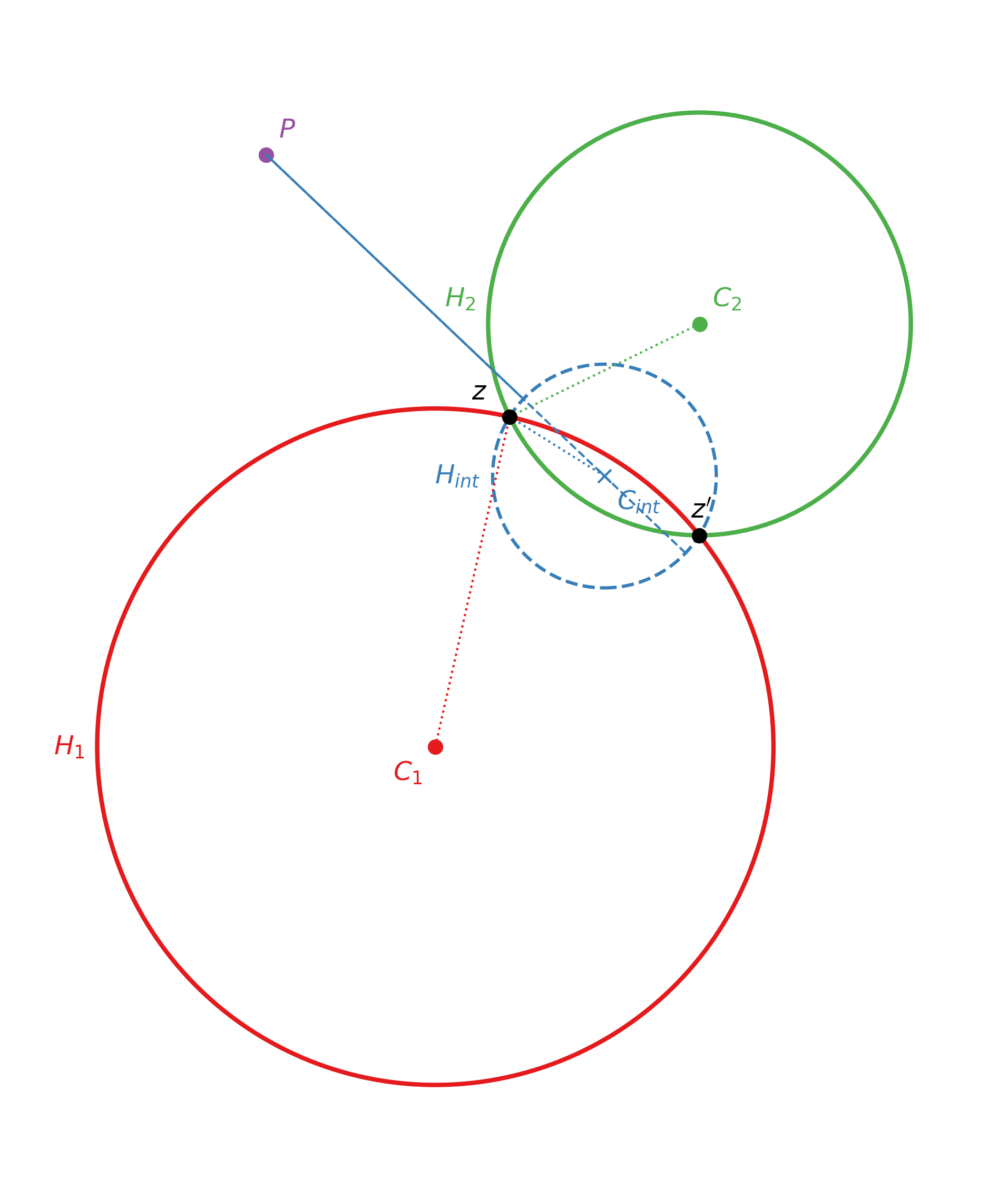}
\caption{Example (in $\mathbb{R}^2$) of a hypersphere {$H_{int}$} containing the intersection between the hypersphere {$H_1$} and {$H_2$}.}\label{fig:hypersphere-intersect}
\end{figure}

We can see the main idea behind our hypersphere intersection approximation on Figure~\ref{fig:hypersphere-intersect} . Each prototype is a center of a hypersphere, and the radius is its distance to the latent feature vector considered $z$. As illustrated, the small sphere $H_{int}$ that approximates the intersection between the two larger spheres leads to a better estimation of the feature vector itself. Furthermore, computing the lower and upper bound of the distance between another prototype $\mathbf{p}$ and the latent vector $z$ becomes trivial ($C_{int} \pm r_{int}$).

\end{document}